\def\year{2016}
\documentclass[letterpaper]{article}
\usepackage{aaai16}
\usepackage{times}
\usepackage{helvet}
\usepackage{courier}

\usepackage{amsmath}    
\usepackage{amsfonts}   
\usepackage{amsmath}    
\usepackage{amssymb}
\usepackage{verbatim}   
\usepackage{color}      
\usepackage[usenames, dvipsnames]{xcolor}
\usepackage{subfig}
\usepackage{tikz}
\usetikzlibrary{fit,calc,positioning,decorations.pathreplacing,matrix}
\usepackage{url}

\usepackage{amsthm}
\usepackage{thmtools, thm-restate}

\newtheorem{theorem}{Theorem}

\newtheorem{lemma}[theorem]{Lemma}
\theoremstyle{definition}
\newtheorem{definition}{Definition}

\theoremstyle{plain}
\newtheorem{proposition}[theorem]{Proposition}
\theoremstyle{definition}
\newtheorem{example}{Example}

\renewcommand{\phi}{\varphi}

\newcommand{\U}{\mathcal{U}}

\newcommand{\cl}[1]{Cl_{#1}}
\newcommand{\maj}{\mathrm{maj}}

\newcommand{\merge}[3]{\Delta^{#1}_{#2}(#3)} 
\newcommand{\distance}[2]{H(#1, #2)} 

\renewcommand{\L}{{\mathcal L}}
\newcommand{\M}{{\mathcal M}}
\newcommand{\N}{{\mathcal N}}
\newcommand{\Lhorn}{\L_{\mathit{Horn}}}

\newcommand{\Lkrom}{\L_{\mathit{Krom}}}
\renewcommand{\Lkrom}{\L_{\mathit{2CNF}}}
\newcommand{\Llits}{\L_{\mathit{1CNF}}}

\newcommand{\horn}{\mathit{Horn}}
\newcommand{\krom}{\mathit{2CNF}}
\newcommand{\lits}{\mathit{1CNF}}

\renewcommand{\mod}{\mathit{Mod}}

\renewcommand{\sum}{{\Sigma}}
\newcommand{\GMAX}{{\mathit{GMax}}}
\newcommand{\GMIN}{{\mathit{GMin}}}
\newcommand{\GMax}{\GMAX}
\newcommand{\GMin}{\GMIN}
\newcommand{\leximax}{{\mathit{leximax}}}
\newcommand{\leximin}{{\mathit{leximin}}}

\renewcommand{\phi}{\varphi}

\begin{document}
%
\title{Distributing Knowledge into Simple Bases}
 \author{Adrian Haret \and Jean-Guy Mailly \and Stefan Woltran\\
 Institute of Information Systems\\TU Wien, Austria\\
 $\{$haret,jmailly,woltran$\}$@dbai.tuwien.ac.at}

\maketitle
\begin{abstract}
\begin{quote}
  Understanding the behavior of belief change operators for fragments
  of classical logic has received increasing interest over the last
  years.  Results in this direction are mainly concerned with adapting
  representation theorems. However, fragment-driven belief change also
  leads to novel research questions.  In this paper we propose the
  concept of belief distribution, which can be understood as the
  reverse task of merging. More specifically, we are interested in the
  following question: given an arbitrary knowledge base $K$ and some
  merging operator $\Delta$, can we find a profile $E$ and a
  constraint $\mu$, both from a given fragment of classical logic,
  such that $\Delta_\mu(E)$ yields a result equivalent to $K$? In
  other words, we are interested in seeing if $K$ can be distributed
  into knowledge bases of simpler structure, such that the task of
  merging allows for a reconstruction of the original knowledge.  Our
  initial results show that merging based on drastic distance allows
  for an easy distribution of knowledge, while the power of
  distribution for operators based on Hamming distance relies heavily
  on the fragment of choice.
\end{quote}
\end{abstract}

\section{Introduction}

Belief change and belief merging have been topics of interest in
Artificial Intelligence for three decades
\cite{AGM85,KM91,KP02}. However, the restriction of such operators to
specific fragments of propositional logic has received increasing
attention only in the last years
\cite{DelgrandeSTW13,CreignouPPW14,CreignouPRW14,ZhuangP12,ZhuangPZ13,ZhuangP14,DelgrandeP15,HaretRW15}. Mostly,
the question tackled in these works is ``How should rationality
postulates and change operators be adapted to ensure that the result
of belief change belongs to a given fragment?''.  Surprisingly, the
question concerning the extent to which the result of a belief change
operation can deviate from the fragment under consideration has been
neglected so far.  In order to tackle this question, we focus here on
a certain form of reverse merging.  The question is, given an
\emph{arbitrary} knowledge base $K$ and some IC-merging (i.e.\ merging
with integrity constraint, see \cite{KP02}), operator $\Delta$ can we
find a profile $E$, {\em i.e.} a tuple of knowledge bases, and a
constraint $\mu$, both from a given \emph{fragment} of classical
logic, such that $\Delta_\mu(E)$ yields a result equivalent to $K$? In
other words, we are interested in seeing if $K$ can be distributed
into knowledge bases of simpler structure, such that the task of
merging allows for a reconstruction of the original knowledge. We call
this operation {\em knowledge distribution}.

Studying the concept of knowledge distribution can be motivated from
different points of view.  First, consider a scenario where the
storage devices have limited expressibility, for instance, databases
or logic programs. Our analysis will show which merging operators are
required to reconstruct arbitrary knowledge stored in such a set of
limited devices.  Second, distribution can also be understood as a
tool to hide information; only users who know the used merging
operator (which thus acts as an encryption key) are able to faithfully
retrieve the distributed knowledge.  Given the high complexity of
belief change (even for revision in ``simple'' fragments like $\horn$
and $\krom$ \cite{EG92,LS01,CreignouPW13}), brute-force attack to
guess the merging operator is unthinkable. Finally, from the
theoretical perspective our results shed light on the power of
different merging operators when applied to profiles from certain
fragments. In particular, our results show that merging $\lits$
formulas via the Hamming-distance based operator $\Delta^{H,\sum}$
does not need additional care, since the result is guaranteed to stay
in the fragment.

\paragraph{Related Work.}
Previous work on merging in fragments of propositional logic proposed
an adaptation of existing belief merging operators to ensure that the
result of merging belongs to a given fragment \cite{CreignouPRW14}, or
modified the rationality postulates in order to function in the
$\horn$ fragment \cite{HaretRW15}.  Our approach is different, since
we do not require that the result of merging stays in a given
fragment.  On the contrary, we want to decompose arbitrary bases into
a fragment-profile.  Recent work by Liberatore has also addressed a
form of meta-reasoning over belief change operators.  In
\cite{Liberatore:2015}, the input is a profile of knowledge bases with
the expected result of merging $R$, and the aim is to determine the
reliability of the bases (for instance, represented by weights) which
allow the obtaining of $R$.
In another paper,
\citeauthor{Liberatore15}
\shortcite{Liberatore15}
identifies, given a sequence of belief revisions and their results, the
initial pre-order which characterizes the revision operator.  Finally,
even if our approach may seem related to {\em Knowledge Compilation}
(KC) \cite{DarwicheM02,FargierM14,Marquis15}, both methods are in fact
conceptually different. KC aims at modifying a knowledge base $K$ into
a knowledge base $K'$ such that the most important queries for a given
application (consistency checking, clausal entailment, model
counting, $\dots$) are simpler to solve with $K'$. Here, we are interested
in the extent to which it is possible to \emph{equivalently represent} an arbitrary knowledge base
by simpler fragments when using merging as a recovery operation. 

\paragraph{Main Contributions.}
We formally introduce the concept of knowledge distributability, as well
as a restricted version of it where the profile is limited to a single
knowledge base (simplifiability). We show that for drastic distance
arbitrary knowledge can be distributed into bases restricted to mostly any kind of
fragment, while simplifiability is limited to trivial cases. On the
other hand, for Hamming-distance based merging the picture is more
opaque.  We show that for $\lits$, distributability w.r.t.\
$\Delta^{H,\sum}$ is limited to trivial cases, while slightly more can
be done with $\Delta^{H,\GMIN}$ and $\Delta^{H,\GMAX}$.  For $\krom$
we show that arbitrary knowledge can be distributed and even be
simplified.  Finally, we discuss the $\horn$ fragment for which the
results for $\Delta^{H,\sum}$, $\Delta^{H,\GMIN}$ and
$\Delta^{H,\GMAX}$ are situated in between the two former fragments.

\section{Background}

\paragraph{Fragments of Propositional Logic.}
We consider $\L$ as the language of propositional logic
over some fixed alphabet $\U$ of propositional atoms.
We use standard connectives $\vee$, $\wedge$, $\neg$,
and constants $\top$, $\bot$. A clause is a disjunction of literals.
A clause is called \emph{Horn} if at most one of its literals is positive.
An interpretation is a set of atoms (those set to true). 
The set of all interpretations is $2^{\U}$.
Models of a formula $\phi$ are denoted by $\mod(\phi)$. A knowledge base (KB) is a finite set of formulas and we identify models of a KB $K$ via
$\mod(K)=\bigcap_{\phi\in K}\mod(\phi)$. A profile is a finite non-empty tuple of KBs.
Two formulae $\phi_1, \phi_2$ (resp. KBs $K_1,K_2$) are equivalent, denoted $\phi_1 \equiv \phi_2$ (resp. $K_1 \equiv K_2$), when they have the same set of models.

We use a rather general and abstract notion of fragments.

\begin{definition}
A mapping $\cl{}: 2^{2^\U}\longrightarrow 2^{2^\U}$ is called \emph{closure-operator}
if it satisfies the following 
for any $\M,\N\subseteq 2^\U$:
\begin{itemize}
  \item If $\M\subseteq \N$, then
$\cl{}(\M)\subseteq\cl{}(\N)$ 
\item If $\vert \M\vert =1$, then $\cl{}(\M)=\M$ 
\item $\cl{}(\emptyset)=\emptyset$.
\end{itemize}
\end{definition}

\begin{definition}
$\L'\subseteq \L$ is called a \emph{fragment} if 
it is closed under conjunction (i.e.,
$\phi \wedge \psi \in \L'$
for any $\phi, \psi \in \L'$),
and there exists an associated closure-operator $\cl{}$ such that 
(1) for all $\psi \in \L'$, $\mod(\psi) = \cl{}(\mod(\psi))$ and
(2) for all $\M\subseteq 2^\U$ 
there is a $\psi \in \L'$  with $\mod(\psi)= \cl{}(\M)$.
We often denote the closure-operator $\cl{}$ associated to a fragment $\L'$ as 
$\cl{\L'}$.
\end{definition}

\begin{definition}
For a fragment $\L'$, we call 
a finite set $K\subseteq \L'$ an
$\L'$-knowledge base. 
An $\L'$-profile is a profile over $\L'$-knowledge bases.
A KB $K'\subseteq \L$ is called $\L'$-\emph{expressible} if there exists an $\L'$-KB $K$,
such that $K'\equiv K$.
\end{definition}

Many well known fragments of propositional logic 
are indeed captured by our notion. 
For the Horn-fragment $\Lhorn$, i.e.\ the set of all conjunctions of Horn clauses over $\U$,
take
the operator
$\cl{\Lhorn}$ defined as the fixed point of the function 
$$
\cl{\Lhorn}^1(\M) = 
\{ \omega_1\cap \omega_2 \mid \omega_1,\omega_2\in \M \}.
$$
The fragment $\Lkrom$ which is restricted to formulas over clauses
of length at most $2$ is linked to the operator
$\cl{\Lkrom}$ defined as the fixed point of the function $\cl{\Lkrom}^1$ given by
$$
\cl{\Lkrom}^1(\M) = 
\{  \maj_3(\omega_1,\omega_2,\omega_3) \mid \omega_1,\omega_2,\omega_3\in \M \}.
$$
Here, we use the ternary majority function $\maj_3(\omega_1,\omega_2,\omega_3)$ which yields an interpretation 
containing those atoms which are true in at least two out of $\omega_1,\omega_2,\omega_3$.
Finally, we are also interested in the $\Llits$ fragment which is just composed of conjunctions of literals;
its associated operator $\cl{\Llits}$ is defined as the fixed point of the function
\begin{eqnarray*}
\cl{\Llits}^1(\M) & = & 
\{ \omega_1\cap \omega_2,\omega_1\cup  \omega_2 \mid \omega_1,\omega_2\in \M \} \cup\\
&& \{ \omega_3 \mid \omega_1\subseteq \omega_3 \subseteq \omega_2; \omega_1,\omega_2\in\M\}.
\end{eqnarray*}
Note that full classical logic is given via the identity closure operator $\cl{\L}(\M)=\M$.

\paragraph{Merging Operators.} 
We focus on {\em IC-merging}, 
where a profile 
is mapped into a KB, such that
the result 
satisfies some integrity
constraint. 
Postulates for IC-merging have been stated
in \cite{KP02}.
We recall a specific family of  IC-merging operators,
based on distances between  interpretations, see also 
\cite{KLM04}.

\begin{definition}
A distance between interpretations is a mapping $d$ from two
interpretations to a non-negative real number, 
such that for
all $\omega_1,\omega_2,\omega_3\subseteq \U$,
(1) $d(\omega_1,\omega_2) = 0$ iff $\omega_1 = \omega_2$;
(2) $d(\omega_1,\omega_2) = d(\omega_2,\omega_1)$; and
(3) $d(\omega_1,\omega_2) +
    d(\omega_2,\omega_3) \geq d(\omega_1,\omega_3)$.
We will use two specific distances:
\begin{description}
  \item[drastic distance] $D(\omega_1,\omega_2) = 1$ if $\omega_1 =
    \omega_2$, $0$ otherwise;
  \item[Hamming distance] $H(\omega_1,\omega_2) = |(\omega_1 \setminus
    \omega_2) \cup (\omega_2 \setminus \omega_1)|$.
\end{description}
\end{definition}

We overload the previous notations to define the distance between 
an interpretation $\omega$ and a KB $K$: if $d$ is a distance between interpretations,
then
\[
d(\omega,K) = \min_{\omega' \in \mod(K)}d(\omega,\omega').
\]
Next, an aggregation function must be used to
evaluate the distance between an interpretation and a profile.

\begin{definition}\label{def:aggregation-function}
An aggregation function 
$\otimes$ 
associates a
non-negative number to every finite tuple of non-negative numbers,
such that:
\begin{enumerate}
  \item 
If $x \leq y$, then
    $\otimes(x_1,\dots,x,\dots,x_n) \leq \otimes(x_1,\dots,y,\dots,x_n)$;
  \item 
$\otimes(x_1,\dots,x_n) = 0$ iff $x_1 =\dots = x_n
    = 0$;
  \item 
For every non-negative number $x$, $\otimes(x) = x$.
\end{enumerate}
As aggregation functions, we will consider the sum $\sum$, and $\GMax$
and $\GMin$\footnote{$\GMax$ and $\GMin$ are also known as $\leximax$ and $\leximin$ respectively. Stricto sensu, these functions return a vector of numbers, and not a single number. However, $\GMax$ (resp. $\GMin$) can be associated with an aggregation function as defined in Definition~\ref{def:aggregation-function} which yields the same vector ordering than $\GMax$ (resp. $\GMin$). We do a slight abuse by using directly $\GMax$ and $\GMin$ as the names of aggregation functions. See \cite{KLM02}.}, defined as follows. Given a profile $(K_1,\dots,K_n)$,
let $V_\omega = (d_1^\omega,\dots,d_n^\omega)$ be the vector of
distances s.t. $d_i^\omega =
d(\omega,K_i)$. $\GMax(d_1^\omega,\dots,d_n^\omega)$
(resp. $\GMin(d_1^\omega,\dots,d_n^\omega)$) is defined by ordering
$V_\omega$ in decreasing (resp. increasing) order. Given two
interpretations $\omega_1,\omega_2$, $\GMax(d_1^{\omega_1},\dots,d_n^{\omega_1})
\leq \GMax(d_1^{\omega_2},\dots,d_n^{\omega_2})$
(resp. $\GMin(d_1^{\omega_1},\dots,d_n^{\omega_1}) \leq
\GMin(d_1^{\omega_2},\dots,d_n^{\omega_2})$) is defined by comparing them
w.r.t. the lexicographic ordering.
\end{definition}

Finally,
let $d$ be a distance, $\omega$ an interpretation  
and $E = (K_1,\dots,K_n)$ a profile. Then,
\[
d^\otimes(\omega,E) = \otimes(d(\omega,K_1),\dots,d(\omega,K_n)).
\]
If there is no ambiguity about the aggregation function $\otimes$, we write $d(\omega,E)$ instead of $d^\otimes(\omega,E)$.


\begin{definition}\label{def:icm} 
For any distance $d$ between interpretations, and any aggregation
function $\otimes$, the merging operator $\Delta^{d,\otimes}$ is a
  mapping from a profile $E$ and a formula $\mu$ to a KB, such that
\[
\mod(\Delta^{d,\otimes}_\mu(E))=\min(\mod(\mu),\leq_E^{d,\otimes}),
\]
with $\omega_1 \leq_E^{d,\otimes} \omega_2$ iff $d^\otimes(\omega_1,E) \leq d^\otimes(\omega_2,E)$.
\end{definition}

When we consider a profile containing a single knowledge base $K$, all aggregation functions are equivalent; we write $\Delta_\mu^{d}(K)$ instead of $\Delta_\mu^{d,\otimes}((K))$ for readability.
For drastic distance, $\GMIN$, $\GMAX$, and $\sum$ are equivalent for arbitrary profiles. Thus, 
whenever we show results
for $\Delta^{D,\sum}$, these carry over to
$\Delta^{D,\GMIN}$ and
$\Delta^{D,\GMAX}$.

\section{Main Concepts and General Results}

We now give the central definition for a knowledge base being distributable 
into a profile from a certain fragment with respect to a given merging operator.

\begin{definition}
Let $\Delta$ be a merging operator, $K\subseteq \L$ be an arbitrary KB, and $\L'$ be a fragment.
$K$ is called $\L'$-\emph{distributable} w.r.t.\ $\Delta$  
if there exists an $\L'$-profile $E$ and a formula $\mu\in\L'$, 
such that $\Delta_\mu(E) \equiv K$.
\end{definition}

\begin{example}\label{ex:1}
Let $\U=\{a,b\}$ and consider 
$K=\{ a\vee b \}$ which we want to
check for $\Lhorn$-distributability w.r.t.\ operator $\Delta^{H,\sum}$.
We have $\mod(K)=\{ \{a\},\{b\},\{a,b\}\}$, thus $K$ is not 
$\Lhorn$-expressible
(note that 
$\cl{\Lhorn}(\mod(K))=\{\emptyset,\{a\},\{b\},\{a,b\}\}\neq\mod(K)$),
otherwise $K$ would be distributable in a simple way (see Proposition~\ref{prop:clear} below).

Take the $\Lhorn$-profile 
$E=(K_1,K_2)$ with 
$K_1 = \{ a \wedge b \}$,
$K_2 = \{ \neg a \vee \neg b \}$,
together with the empty constraint $\mu=a \vee \neg a$.
We have 
$\mod(K_1)=\{\{a,b\}\}$,
$\mod(K_2)=\{\{a\},\{b\},\emptyset\}$. 
In the following matrix, each line corresponds to the distance between a model of $\mu$ and a KB from the profile $E$ (columns $K_1$ and $K_2$), or between a model of $\mu$ and the profile 
using the sum-aggregation over the distances to the single KBs
(column $\sum$).
\[
\begin{array}{clllll}
& K_1 & K_2  & \sum\\
\{a,b\}	 	        &	0	&		1   &  1\\
\{a\}           	&	1	&		0   &  1\\
\{b\}           	&	1	&		0   &  1\\
\emptyset             	&	2	&		0   &  2\\
\end{array}
\]
We observe that $\mod(\Delta^{H,\sum}_\mu(E))=\{ \{a\},\{b\},\{a,b\}\}$, thus
$\Delta^{H,\sum}_\mu(E) \equiv K$ as desired.
It is easily checked that also other aggregations work: 
$\Delta^{H,\GMAX}_\mu(E)  \equiv
\Delta^{H,\GMIN}_\mu(E) 
\equiv K$.
\hfill$\diamond$
\end{example}

Next, we recall that IC-merging of a single KB yields revision. 
Thus, the concept we introduce next is also of interest, as it represents
a certain form of reverse revision.

\begin{definition}
Let $\Delta$ be a merging operator, $K\subseteq \L$ an arbitrary KB, 
and $\L'$ a fragment.
$K$ is called $\L'$-\emph{simplifiable} w.r.t.\ $\Delta$  
if there exists an $\L'$-KB $K'$  and $\mu\in\L'$, 
such that $\Delta_\mu(K') \equiv K$.
\end{definition}

As we will see later, 
the KB $K$ from Example~\ref{ex:1} cannot be $\Lhorn$-simplified w.r.t.\ $\Delta^{H}$; in other words,
we need here at least two KBs to ``express'' $K$.
However, it is rather straightforward that any $\L'$-expressible KB can be $\L'$-simplified. 

\begin{proposition}\label{prop:clear}
For every fragment $\L'$ and
every KB $K$,
it holds that $K$ is $\L'$-simplifiable (and thus also $\L'$-distributable) w.r.t.\ $\Delta$,
whenever $K$ is $\L'$-expressible.
\end{proposition}
\begin{proof}
Let $K'$ be an $\L'$-KB equivalent to $K$,
and let $\mu=(\bigwedge_{\phi \in K'} \phi)$. 
Thus, $\mu \in\L'$ by definition of fragments and it is easily verified
that $\Delta_\mu(K') \equiv K$.
\end{proof}

Next, we show that 
in order
to determine whether a KB $K$ is $\L'$-distributable, 
it is sufficient to consider constraints $\mu$ such that 
$\mod(\mu)=\cl{\L'}(\mod(K))$.

\begin{proposition}\label{prop:mu}
  Let $K \in \L$ be a KB, $\L'$ be a
  fragment, $E$ an $\L'$-profile and $\mu\in\L'$.  Then 
	$\Delta_\mu(E)\equiv K$ 
	implies 
	$\Delta_{\mu'}(E)\equiv K$ for any $\mu'$ such that
 	$\mod(\mu') = \cl{\L'}(\mod(K))$.
\end{proposition}

\begin{proof}
%
Let $\Delta=\Delta^{d,\otimes}$.
By Definition~\ref{def:icm},
$\mod(K) = \min(\mod(\mu),\leq^{d,\otimes}_E)$,
hence
$\mod(K) \subseteq \mod(\mu)$.
Moreover, $\mu$ is $\L'$-closed, so 
$\cl{\L'}(\mod(K))=\mod(\mu')
\subseteq \mod(\mu)$.  
We get $\mod(K)\subseteq \mod(\mu') \subseteq \mod(\mu)$.
Thus,
$\mod(K) = \min(\mod(\mu'),\leq^{d,\otimes}_E)$,
i.e.\
	$\Delta_{\mu'}(E)\equiv K$.
\end{proof}

Next, we give two positive results for distributing knowledge
in any fragment. The key idea is to use KBs in the profile 
which have exactly one model (our notion of fragment guarantees existence 
of such KBs). The first result is independent of the distance notion but requires
$\GMIN$ as the aggregation function. The second result is for drastic distance and  thus
works for any of the aggregation functions we consider.

\begin{theorem}\label{thm:gmin}
Let $d$ be a distance and 
$\L'$  be a fragment.
Then for
every KB $K$, such that for
  all distinct $\omega_1,\omega_2 \in \mod(K)$, $d(\omega_1,\omega_2)=e$ for some 
  $e>0$, it holds that $K$ is $\L'$-distributable w.r.t. $\Delta^{d,\GMin}.$
\end{theorem}

\begin{proof}
  Build the $\L'$-profile $E$ such that for each $\omega \in \mod(K)$,
  there is a KB with $\omega$ as its only model.  Thus
  all models of $K$ get a $\GMIN$-vector $(0,e,e,e,e,\ldots)$.  All
  interpretations from $\cl{\L'}(\mod(K))\setminus \mod(K)$ get a
  vector $(f,g,\ldots)$ with $f>0$. Hence, we have
  $\min(\mod(\mu),\leq_E^{d,\GMin}) = \mod(K)$ using $\mu\in\L'$ with
$\mod(\mu) = \cl{\L'}(\mod(K))$.
\end{proof}



\begin{theorem}
For every fragment $\L'$ and 
every knowledge base $K$, 
it holds that $K$ is $\L'$-distributable w.r.t.\ $\Delta^{D,\oplus}$, for $\oplus \in \{\sum,\GMin,\GMax\}$.
\end{theorem}
\begin{proof}
  Given a fragment $\L'$, we take $E=\{ K_\omega \mid \omega\in
  \mod(K) \}$ where $K_\omega\in \L'$ is a knowledge base with single
  model $\omega$ (such $K_\omega\in \L'$ exists due to our definition
  of fragments), and let $\mu$ be such that
  $\mod(\mu)=\cl{\L'}(\mod(K))$; hence also $\mu\in\L'$.  Let
  $\omega'\in\mod(\mu)$ and $n=|\mod(K)|$, we observe that
  $\sum_{K_\omega \in E} H(\omega',K_\omega) = n-1$ when $\omega'\in
  \mod(K)$, and $n$ otherwise.  Thus, $\Delta_\mu^{D,\sum}(E) \equiv
  K$. The same result holds for $\Delta_\mu^{D,\GMax}$ and $\Delta_\mu^{D,\GMin}$.
\end{proof}

%
%

%
Concerning simplifiability w.r.t.\ drastic distance based operators, 
Proposition
\ref{prop:clear} cannot be improved.

\begin{theorem}\label{thm:d}
For every fragment $\L'$ and
every KB $K$,
$K$ is $\L'$-\emph{simplifiable} w.r.t.\ $\Delta^D$ 
iff $K$ is $\L'$-expressible.
\end{theorem}
\begin{proof}
The if-direction is by Proposition \ref{prop:clear}.
For the other direction,
suppose $K$ is not $\L'$-expressible. 
We show that for any $\L'$-KB $K'$,
$\Delta^D_\mu(K')\not\equiv K$ with 
$\mu=\cl{\L'}(K)$. By Proposition~\ref{prop:mu} the result then follows.
Now suppose there exists an $\L'$-KB $K'$ such that 
$\Delta^D_\mu(K') \equiv K$. 
First observe that since $K$ is not $\L'$-expressible, 
$\mod(\mu)\supset \mod(K)$.
Since 
we are working with drastic distance, in order to promote models of $K$, 
we also need them in $K'$, hence $\mod(K')\supseteq \mod(K)$ and since $K'$ is from $\L'$
we have $\mod(K')\supseteq \cl{\L'}(K) = \mod(\mu)$.
Thus there exists $\omega\in\cl{\L'}(\mod(K))\setminus \mod(K)$ having distance $0$ to $K'$,
and thus $\omega\in\Delta^D_\mu(K')$. Since $\omega\notin\mod(K)$, this yields
a contradiction to 
$\Delta^D_\mu(K') \equiv K$. 
\end{proof}

\section{Hamming Distance and Specific Fragments}
We first consider the simplest fragment under consideration, namely
conjunction of literals. As it turns out, (non-trivial) distributability 
for this
fragment w.r.t.\ $\Delta^{H,\Sigma}$ is not achievable. 
We then see that more general fragments allow for
non-trivial distributions. In particular, we show that every 
KB is distributable (and even simplifiable) in the
$\krom$ case,
and we finally give a few 
observations
for $\Lhorn$.

\subsection{The 1CNF Fragment}

The following technical result is important to prove the main result in this section. 



\begin{lemma}\label{cor:cap-cup-equality}
	For any $\Llits$-profile $E = (K_1, \dots, K_n)$ 
and interpretations $\omega_1, \omega_2$, it holds that:
		$$
			\distance{\omega_1}{E} + \distance{\omega_2}{E} = \distance{\omega_1 \cap \omega_2}{E} + \distance{\omega_1 \cup \omega_2}{E}.
		$$
\end{lemma}
\begin{proof}
	It suffices to show that for each $K_i$ in profile $E$, 
	$
	\distance{\omega_1}{K_i} + \distance{\omega_2}{K_i} = \distance{\omega_1\cap \omega_2}{K_i} + \distance{\omega_1 \cup \omega_2}{K_i}
	$.
	Indeed, 
	summing up these equalities over all $K_i \in E$, we get
	\begin{eqnarray*}
	&\sum_{K_i \in E}\distance{\omega_1}{K_i} + \sum_{K_i \in E}\distance{\omega_2}{K_i} =\\ 
	&\sum_{K_i \in E}\distance{\omega_1\cap \omega_2}{K_i} + \sum_{K_i \in E}\distance{\omega_1 \cup \omega_2}{K_i}.
	\end{eqnarray*}	
	Since $\distance{\omega}{E} = \Sigma_{K_i \in E} \distance{\omega}{K_i}$, for any interpretation $\omega$, our conclusion then follows immediately.

Thus,
take $\omega_1', \omega_2'$ to be two interpretations that are closest to $\omega_1$ and $\omega_2$, respectively, among the models of $\mod(K_i)$. 
In other words, $H(\omega_1, \omega_1') = \min_{\omega \in \mod(K_i)} H(\omega_1, \omega)$ and $H(\omega_2, \omega_2') = \min_{\omega \in \mod(K_i)} H(\omega_2, \omega)$. 
By induction on the number of propositional atoms in $\L$, we can show that $\omega_1' \cap \omega_2'$ and $\omega_1' \cup \omega_2'$ are closest in $\mod(K_i)$ to $\omega_1 \cap \omega_2$ and $\omega_1 \cup \omega_2$, respectively.
Thus, we have that $\distance{\omega_1}{K_i} = H(\omega_1, \omega_1')$, $\distance{\omega_2}{K_i} = H(\omega_2, \omega_2')$, $\distance{\omega_1 \cap \omega_2}{K_i} = H(\omega_1 \cap \omega_2, \omega_1' \cap \omega_2')$, $\distance{\omega_1 \cup \omega_2}{K_i} = H(\omega_1 \cup \omega_2, \omega_1' \cup \omega_2')$, 
and our problem reduces to showing that 	$H(\omega_1, \omega_1') + H(\omega_2, \omega_2') = H(\omega_1 \cap \omega_2, \omega_1' \cap \omega_2') + H(\omega_1 \cup \omega_2, \omega_1' \cup \omega_2')$.
By using induction on the number of propositional atoms in $\L$ again, we can show that this equality holds. The argument runs as follows: in the base case, when the alphabet consists of just one propositional atom, the equality is shown to be true by checking all the cases.
For the inductive step we assume the claim holds for an alphabet of size $n$ and show that it also holds for an alphabet of size $n+1$.
More concretely, we analyze the way in which the Hamming distances between interpretations change when we add a propositional atom to the alphabet. 
An analysis of all the possible cases shows that the equality holds.
\end{proof}

Next we observe certain patterns of interpretations that indicate whether a KB is $\Llits$-expressible or not.

\begin{definition}\label{def:critical-pair}
If $K$ is a knowledge base, then a pair of interpretations $\omega_1$ and $\omega_2$ are called \textit{critical with respect to $K$} if $\omega_1 \nsubseteq \omega_2$ and $\omega_2 \nsubseteq \omega_1$, and one of the following cases holds:
	\begin{enumerate}
	\item $\omega_1, \omega_2 \in \mod(K)$ and $\omega_1 \cap \omega_2, \omega_1 \cup \omega_2 \notin \mod(K)$, 
	\item $\omega_1, \omega_2, \omega_1 \cap \omega_2 \in \mod(K)$ and $\omega_1 \cup \omega_2 \notin \mod(K)$,
	\item $\omega_1, \omega_2, \omega_1 \cup \omega_2 \in \mod(K)$ and $\omega_1\cap \omega_2 \notin \mod(K)$,
	\item $\omega_1 \cap \omega_2, \omega_1 \cup \omega_2 \in \mod(K)$ and $\omega_1, \omega_2 \notin \mod(K)$, or
	\item $\omega_1, \omega_1 \cap \omega_2, \omega_1 \cup \omega_2 \in \mod(K)$ and $\omega_2 \notin \mod(K)$.
	\end{enumerate}
\end{definition}

\begin{lemma}\label{lemma:critical-pairs}
	If a KB $K$ is not $\Llits$-expressible, then there exist 
$\omega_1, \omega_2 \in \cl{\Llits}(K)$ being critical with respect to $K$.
\end{lemma}
\begin{proof}
	The fact that $K$ is not $\Llits$-expressible implies that either: (i) $K$ is not closed under intersection or union, or (ii) there are $w_1, w_2, w_3 \in \cl{\Llits}(K)$ such that $w_1 \subseteq w_3 \subseteq w_2$, and $w_1, w_2 \in \mod(K)$, $w_3 \notin \mod(K)$. Case (i) implies that there exist $w_1, w_2 \in \mod(K)$ such that one of Cases 1-3 from Definition~\ref{def:critical-pair} holds. If we are in Case (ii), then consider the interpretation $w_4 = (w_2 \backslash w_3) \cup w_1$. Clearly, $w_1 \subseteq w_4 \subseteq w_2$, hence $w_4 \in \cl{\Llits}(K)$. Also, $w_3 \cap w_4 = w_1$ and $w_3 \cup w_4 = w_2$. There are two sub-cases to consider here. If $w_4 \notin \mod(K)$, then we are in Case~4 of Definition~\ref{def:critical-pair}. If $w_4 \in \mod(K)$, then we are in Case~5 of Definition~\ref{def:critical-pair}.
\end{proof}

\begin{example}
Let us consider the KB $K$ such that $\mod(K) = \{\emptyset,\{a\},\{b\},\{c\},\{a,c\},\{b,c\},\{a,b,c\}\}$. $K$ is not $\lits$-expressible; indeed, $\cl{\lits}(\mod(K)) = \mod(K) \cup \{\{a,b\}\}$.

Here, we identify several sets of critical interpretations w.r.t. $K$.
First, $S_1 = \{\{a,c\}, \{a,b\}, \{a\}, \{a,b,c\}\}$ corresponds to the situation described in Case 5 of Definition 9, with $\omega_1 = \{a,c\}$ and $\omega_2 = \{a,b\}$.

The set $S_2 = \{\{b,c\},\{a,b\}, \{b\}, \{a,b,c\}\}$ also corresponds to Case 5, with $\omega_1 = \{b,c\}$ and $\omega_2 = \{a,b\}$.

We can also consider the set of interpretations $S_3 = \{\emptyset, \{a\}, \{b\}, \{a,b\}\}$, which corresponds to Case 2 of Definition 9, with $\omega_1 = \{a\}$ and $\omega_2 = \{b\}$.
The models of $K$ and the sets of critical interpretations are represented in Figure~\ref{fig:critical-interpretations}.

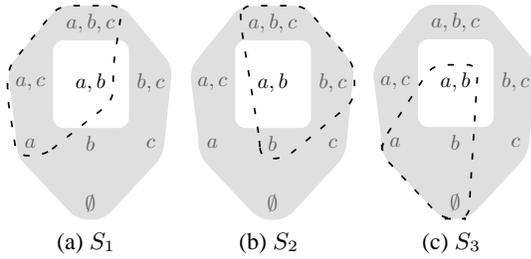
\begin{figure}[h]
\centering
\subfloat[$S_1$]{
\scalebox{0.8}{
\begin{tikzpicture}[auto,node distance=1cm,thick]
\node (abc) {$a,b,c$};
\node (ab) [below of=abc] {$a,b$};
\node (ac) [left of=ab] {$a,c$};
\node (bc) [right of=ab] {$b,c$};
\node (a) [below of=ac] {$a$};
\node (b) [below of=ab] {$b$};
\node (c) [below of=bc] {$c$};
\node (empty) [below of=b] {$\emptyset$};

 \fill[color=lightgray,opacity=0.5,rounded corners=5pt,even odd rule] (a.west) -- (ac.north west) -- (abc.north west) -- (abc.north)
              -- (abc.north east) -- (bc.north east) -- (c.south east)
              -- (empty.south east) -- (empty.south)
              -- (empty.south west) -- (a.south west) -- cycle
              (abc.south) -| (bc.west) |- (b.north) -| (ac.east) |- (abc.south);

\draw[loosely dashed,rounded corners=5pt] (a.south east) -- (ab.south east) -- (ab.north east)
                      -- (abc.north east) -- (abc.north west)
                      -- (ac.north west) -- (ac.south west) -- (a.south west)
                      -- cycle;

\end{tikzpicture}
}}
\subfloat[$S_2$]{
\scalebox{0.8}{
\begin{tikzpicture}[auto,node distance=1cm,thick]
\node (abc) {$a,b,c$};
\node (ab) [below of=abc] {$a,b$};
\node (ac) [left of=ab] {$a,c$};
\node (bc) [right of=ab] {$b,c$};
\node (a) [below of=ac] {$a$};
\node (b) [below of=ab] {$b$};
\node (c) [below of=bc] {$c$};
\node (empty) [below of=b] {$\emptyset$};

 \fill[color=lightgray,opacity=0.5,rounded corners=5pt,even odd rule] (a.west) -- (ac.north west) -- (abc.north west) -- (abc.north)
              -- (abc.north east) -- (bc.north east) -- (c.south east)
              -- (empty.south east) -- (empty.south)
              -- (empty.south west) -- (a.south west) -- cycle
              (abc.south) -| (bc.west) |- (b.north) -| (ac.east) |- (abc.south);

\draw[loosely dashed,rounded corners=5pt]  (b.west)
              -- (abc.west) |- (abc.north east)
              -- (bc.north east) -- (bc.south east)
              -- (b.south east) -- (b.south west) -- cycle;
\end{tikzpicture}
}}
\subfloat[$S_3$]{
\scalebox{0.8}{
\begin{tikzpicture}[auto,node distance=1cm,thick]
\node (abc) {$a,b,c$};
\node (ab) [below of=abc] {$a,b$};
\node (ac) [left of=ab] {$a,c$};
\node (bc) [right of=ab] {$b,c$};
\node (a) [below of=ac] {$a$};
\node (b) [below of=ab] {$b$};
\node (c) [below of=bc] {$c$};
\node (empty) [below of=b] {$\emptyset$};

 \fill[color=lightgray,opacity=0.5,rounded corners=5pt,even odd rule] (a.west) -- (ac.north west) -- (abc.north west) -- (abc.north)
              -- (abc.north east) -- (bc.north east) -- (c.south east)
              -- (empty.south east) -- (empty.south)
              -- (empty.south west) -- (a.south west) -- cycle
              (abc.south) -| (bc.west) |- (b.north) -| (ac.east) |- (abc.south);

\draw[loosely dashed,rounded corners=5pt]  (a.west) -- (ab.north west)
              -- (ab.north east)
              -- (empty.south east) -- (empty.south)
              -- (empty.south west) -- (a.south west) -- cycle;
\end{tikzpicture}
}}
\caption{Models of $K$ are in the shaded area; critical interpretations are in the dashed areas. \label{fig:critical-interpretations}}
\end{figure}
\end{example}

We can now state the central result of this section.

\begin{theorem}
A KB $K$ is $\Llits$-distributable with respect to $\Delta^{H, \Sigma}$ if and only if $K$ is $\Llits$-expressible.
\end{theorem}
\begin{proof} 
\textit{If part.} By Proposition~\ref{prop:clear}.

\textit{Only if part.} Let $K$ be a KB that is not $\Llits$-expressible. We will show that it is not $\Llits$-distributable w.r.t.\ $\Delta^{H,\Sigma}$. Suppose, on the contrary, that $K$ is $\Llits$-distributable. Then there exists an $\Llits$ profile $E = (K_1, \dots, K_n)$ such that $\merge{H, \Sigma}{\mu}{E} \equiv K$, where $\mod(\mu) = \cl{\Llits}(\mod(K))$ (cf.\ Proposition~\ref{prop:mu}).

By Lemma~\ref{lemma:critical-pairs}, there exist interpretations $\omega_1, \omega_2 \in \mod(\mu)$ that are critical with respect to $K$. By Lemma~\ref{cor:cap-cup-equality}, we have
\begin{equation}\label{eq:cap-cup-eq}
\!
	\distance{\omega_1}{E} {+} \distance{\omega_2}{E} {=} \distance{\omega_1 {\cap} \omega_2}{E} {+} \distance{\omega_1 {\cup} \omega_2}{E}.\!
\end{equation}
Let us now do a case analysis depending on the type of critical pair we are dealing with. If we are in Case~1 of Definition~\ref{def:critical-pair}, then it needs to be the case that $\distance{\omega_1}{E} = \distance{\omega_2}{E} = m$, $\distance{\omega_1 \cap \omega_2}{E} = m+k_1$ and $\distance{\omega_1 \cup \omega_2}{E}=m+k_2$, for some integers $m \geq 0$ and $k_1, k_2 > 0$. Plugging these numbers into Equality~(\ref{eq:cap-cup-eq}), we get that $2m = 2m + k_1 + k_2$ and $k_1 + k_2 = 0$. Since $k_1, k_2 > 0$, we have arrived at a contradiction. 
If we are in Case~2, then it needs to be the case that ${\distance{\omega_1 \cap \omega_2}{E}} = \distance{\omega_1 \cup \omega_2}{E} = m$, $\distance{\omega_1}{E} = m+k_1$ and $\distance{\omega_2}{E}=m+k_2$, for some integers $m \geq 0$ and $k_1, k_2 > 0$. Plugging these numbers into Equality~(\ref{eq:cap-cup-eq}) again, we get a contradiction along the same lines as in Case~1. 
If we are in Case~3, then it needs to hold that $\distance{\omega_1}{E} = \distance{\omega_1 \cap \omega_2}{E} = \distance{\omega_1 \cup \omega_2}{E} = m$, $\distance{\omega_2}{E} = m+k$, for some integers $m \geq 0$ and $k > 0$. Plugging these numbers into Equality~(\ref{eq:cap-cup-eq}) gives us $2m + k = 2m$ and hence $k = 0$. Since $k > 0$, we have arrived at a contradiction. Cases~4 and~5 are entirely similar.
\end{proof}

In other words, for any $\Llits$-profile and $\mu\in\lits$, 
$\Delta^{H,\sum}_\mu$ is guaranteed to be $\Llits$-expressible as well.
As we have already shown in Theorem~\ref{thm:gmin}, this is not necessarily the case
if we replace $\sum$ by $\GMIN$. 
The following example shows how to obtain a similar behavior for $\GMAX$; we then generalize
this idea below.

\begin{example}
Let 
$\U=\{a,b\}$ and 
$K=\{ a\vee b, \neg a \vee \neg b\}$.  We have 
$\mod(K)=\{\{a\},\{b\}\}$. $K$ is not $\Llits$-expressible, since 
$\cl{\Llits}(\mod(K))=2^\U$. Let $K_{S}$ be the $\Llits$-KB with a single 
model $S$ for any $S\subseteq \U$ and let us have a look at the following distance matrix 
for $\mu$ with $\mod(\mu)=\cl{\Llits}(\mod(K))$, 
$E=(K_{\{a\}},K_{\{b\}})$,
and 
$E'=(K_{\emptyset},K_{\{a,b\}})$.
\begin{small}
$$
\begin{array}{lcccccc}
&
\!\!K_{\emptyset} &
\!K_{\{a\}} &
\!K_{\{b\}} &
\!K_{\{a,b\}} &
\!\!H^\GMIN(E) & 
\!\!H^\GMAX(E') \\
\emptyset & \!\!0 & 1 & 1 & 2 & \!(1,1) &\! (2,0) \\
\{a\} & \!\!1 &  0 & 2 & 1 & \!(0,2) & \!(1,1) \\
\{b\} & \!\!1 & 2 & 0 & 1 & \!(0,2) & \!(1,1) \\
\{a,b\} & \!\!2 & 1 & 1 & 0 & \!(1,1) & \!(2,0)
\end{array}
$$
\end{small}

Recall that the lexicographic order of the involved vectors is $(0,2)<(1,1)<(2,0)$.
We thus get that
$\Delta^{H,\GMIN}_\mu(E)\equiv K$ (see also Theorem~\ref{thm:gmin}), and
on the other hand, 
$\Delta^{H,\GMAX}_\mu(E')\equiv K$.
\hfill$\diamond$
\end{example}

\begin{theorem}
Any KB $K$ such that $\mod(K)=\{\omega, \omega'\}$ is $\Llits$-distributable with respect to $\Delta^{H, \GMAX}$.
\end{theorem}
\begin{proof}
If $K$ is $\Llits$-expressible, then the conclusion follows from Proposition~\ref{prop:clear}. 
If $K$ is not $\Llits$-expressible, then consider the set $\cl{\Llits}(\mod(K))\backslash\mod(K)=\{\omega_1, \dots, \omega_n\}$.
We define the profile $E=(K_1, \dots, K_n)$, where $\mod(K_i)=\{\U\backslash \omega_i\}$, for $i \in \{1, \dots, n\}$.
We show that $\merge{H, \GMAX}{\mu}{E} \equiv K$, where $\mod(\mu) = \cl{\Llits}(\mod(K))$.

First,  we have that $H(\omega_i, \U\backslash \omega_i)=|\U|$, which implies that $H^\GMAX(\omega_i, E)=\GMAX(|\U|, \dots)$, for any $i \in \{1, \dots, n\}$. 
Furthermore, since $H(\omega, \U \backslash \omega_i) < |\U|$ and $H(\omega', \U \backslash \omega_i)<|\U|$, for any $i \in \{1, \dots, n\}$, it follows that $\omega <_E^{H, \GMAX} \omega_i$ and $\omega' <_E^{H, \GMAX} \omega_i$. 
Next, we show that $H^\GMAX(\omega, E)=H^\GMAX(\omega', E)$.

Consider the vectors $V = (H(\omega, \omega_1), \dots, H(\omega, \omega_n))$ and $V' = (H(\omega', \omega_1), \dots, H(\omega', \omega_n))$. 
Our claim is that $\GMAX(V)=\GMAX(V')$.
To see why, notice that the elements in $\cl{\Llits}(\mod(K))$ form a complete subset lattice with $\omega\cup \omega'$ and $\omega\cap \omega'$ as the top and bottom elements, respectively. 
Let us write $H(\omega,\omega')=m$.
This lattice has $2^m$ elements, and the maximum distance of two elements in it is $m$. 
Thus, the vector $V$ is the vector of distances between $\omega$ and every other element in this lattice, except itself and $\omega'$.
A similar consideration holds for $V'$.
Hence $V$ and $V'$ are vectors of length $2^{m-2}$ whose elements are $m-1,m-2, \ldots, 1$.
We can actually count how many times each number appears in $V$ and $V'$.
The number of interpretations in the lattice that are at distance of 1 from $\omega$ (and $\omega'$) is $\binom{1}{m}$: thus, $m-1$ appears $\binom{1}{m}$ times in $V$ (and $V'$). 
The number of interpretations that are at distance 2 from $\omega$ (and $\omega'$) is $\binom{2}{m}$, thus $m-2$  appears  $\binom{2}{m}$ times in $V$ and $V'$. 
We iterate this argument for every distance, up to 1. 
It is then easy to see that, based on these considerations, $V$ and $V'$ are equal when sorted in descending order.
Our conclusion follows from this.
\end{proof}	

\subsection{The 2CNF Fragment}

We show that every knowledge base $K$ can be distributed in the fragment $\Lkrom$. Even a single $\Lkrom$ knowledge base is enough to represent $K$.
Before giving the general result, we sketch the idea via an example.

\begin{example}
Let $K$ be a KB with
$\mod(K)=\{\{a,b\},\{b,c,e\},\{a,c,d\}\}$. We observe that $K$ is 
not $\Lkrom$-expressible since $\cl{\Lkrom}(\mod(K))=\mod(K)\cup \{a,b,c\}$.
However, we can give an $\Lkrom$-KB $K'$ using three new atoms
$x,y,z$ to penalize the undesired interpretation $\{a,b,c\}$
such that $\Delta^H_\mu(K')\equiv K$,
with $\mu\in\Lkrom$ of the form $\mod(\mu)=\cl{\Lkrom}(\mod(K))$.
To this end, assume $K'$ with
$\mod(K') = \{ \omega_1,\omega_2,\omega_3,\omega_4\}$
of the form
\begin{eqnarray*}
\omega_1 &=& \{a,b,x,y\}, \\
\omega_2 &=& \{b,c,e,x,z\}, \\
\omega_3 &=& \{a,c,d,y,z\}, \\
\omega_4 &=& \{a,b,c,x,y,z\}.
\end{eqnarray*}
One can verify that $\cl{\Lkrom}(K')=\mod(K')$.
Thus, $K'$ can be picked from $\Lkrom$.
We use $\mu$ such that $\mod(\mu)=\cl{\Lkrom}(\mod(K))$ and get 
distances
$$
\begin{array}{lcccccc}
&
\omega_1 &
\omega_2 &
\omega_3 &
\omega_4 &
\min
\\
\{a,b\} 	& 2 & 5 & 5 & 4 & 2
\\
\{b,c,e\} 	& 4 & 2 & 6 & 4 & 2
\\
\{a,c,d\} 	& 4 & 6 & 2 & 4 & 2
\\
\{a,b,c\}	& 3 & 4 & 4 & 3 & 3
\end{array}
$$
Here, each line gives the distance between a model of $\mu$ and a model of $K'$ ($\omega_i$ columns), or between a model of $\mu$ and $K'$ ($\min$ column).
The key observation is that pairs from $x,y,z$ as used in $\omega_1$, $\omega_2$, $\omega_3$
give minimal distances $2$ while the remaining interpretation $\omega_4$, 
which corresponds to the closure of $K$, contains all three new
atoms (since $maj_3(\{x,y\},\{x,z\},\{y,z\})=\{x,y,z\}$).\hfill$\diamond$
\end{example}

\begin{theorem}\label{theorem:krom-representability}
Any KB 
$K$
is $\Lkrom$-simplifiable w.r.t.~$\Delta_{\mu}^H$. 
\end{theorem}

\begin{proof}
We have to show that 
for any KB $K$, there exists an $\Lkrom$-KB $K'$ and a formula $\mu \in \Lkrom$ such that $\Delta_{\mu}^H(K') \equiv K$.
If $K$ is $\Lkrom$-expressible, the result is due to Proposition~\ref{prop:clear}. So suppose that $K$ is not $\Lkrom$-expressible and
let $\mod(K) = \{\omega_1,\dots,\omega_n\}$.
Consider a set of new atoms $A = \{a_1,\dots,a_n\}$, and for each
$\omega_i \in \mod(K)$, let $\omega_i' = \omega_i \cup A \setminus \{a_i\}$. 
We define the $\Lkrom$-KB $K'$ and $\mu\in\Lkrom$ such that 
\begin{eqnarray*}
\mod(K') & = & \cl{\Lkrom}(\{\omega_i' \mid \omega_i \in \mod(K)\}) \\
\mod(\mu) & = & \cl{\Lkrom}(\mod(K)).
\end{eqnarray*}
Let $\Omega'=\{\omega_i' \mid \omega_i \in \mod(K)\}$.
We first show that for each 
$\omega\in\mod(K')\setminus\Omega'$,
$A\subseteq \omega$. 
Indeed, for any triple
$\omega_j,\omega_k,\omega_l\in\mod(K)$, such that 
$\omega_{jkl}=\maj_3(\omega_j,\omega_k,\omega_l)\notin\mod(K)$, we observe that
$\maj_3(\omega'_j,\omega'_k,\omega'_l)=
\omega_{jkl}\cup \maj_3(
A\setminus\{a_j\},
A\setminus\{a_k\},
A\setminus\{a_l\})=
\omega_{jkl}\cup A$. 
Thus, for each
$\omega\in\cl{\Lkrom}^1(\Omega')\setminus\Omega'$,
$A\subseteq \omega$. 
Recall that 
$\mod(K')  =  \cl{\Lkrom}(\Omega')$.
It follows quite easily that
each further interpretation
$\omega\in\cl{\Lkrom}(\Omega')\setminus
(\cl{\Lkrom}^1(\Omega')\cup \Omega')$, also satisfies
$A\subseteq \omega$.

This shows that each model of $K'$ contains at least $n-1$ atoms from $A$.
Thus, for every model $\omega_i\in K$, 
$H(\omega_i,K') = H(\omega_i,\omega_i') = n - 1$.
It remains to show that
for each 
$\omega \in \mod(\mu) \setminus \mod(K)$, 
$H(\omega,K') \geq n$. 
First, let $\omega'\in\Omega'$.
Since $\omega \notin \mod(K)$, $\omega'\setminus A\neq\omega$
and since $\omega'$ contains $n-1$ elements from $A$,
we have
$H(\omega,\omega') \geq n$. 
As shown above all other interpretations $\omega''\in
\mod(K')\setminus\Omega'$ contain all $n$ atoms from $A$,
thus
$H(\omega,\omega'') \geq n$, too.
%
%
\end{proof}

As an immediate consequence, we obtain that any KB $K$ is $\Lkrom$-distributable w.r.t.\ $\Delta^{H,\otimes}$ for \emph{any} aggregation function $\otimes$. 
Note that this result is in strong contrast to the $\Llits$ fragment, 
where only $\Llits$-expressible KBs are $\Llits$-distributable w.r.t.\ $\Delta^{H,\sum}$. 

\subsection{The Horn-Fragment}

We now turn our attention to the $\Lhorn$ fragment. 
Recall Example~\ref{ex:1} where we 
have shown how to distribute some non $\Lhorn$-expressible KB 
using a profile over two $\Lhorn$-KBs.
Our first result shows that in this example case we cannot reduce to profiles of a single KB, i.e.\ that there are KBs which are
$\Lhorn$-distributable but not $\Lhorn$-simplifiable.


\begin{proposition}\label{prop:neg}
Given  a KB $K$ with $\mod(K)=\{\omega_1,\omega_2,\omega_3\}$, where $\omega_3=\omega_1\cup \omega_2$, $H(\omega_1, \omega_2)=2$ and $\omega_1,\omega_2$ are incomparable. Then $K$ is not $\Lhorn$-simplifiable w.r.t. $\Delta^H$.
\end{proposition}

\begin{proof}
The situation described in the Proposition corresponds to
$K = \{\omega \cup \{a\},\omega \cup \{b\}, \omega \cup \{a,b\}\}$ with $\omega$ some interpretation which does not contain $a$ or $b$.  We need $\mod(\mu)=\{\omega,\omega \cup \{a\},\omega \cup \{b\}, \omega \cup \{a,b\}\}$, as required by Proposition~\ref{prop:mu}.
We want to identify a $\Lhorn$-KB $K'$ such that $\Delta_\mu^H(K') \equiv K$. This means that $\omega$ is the single model of $\mu$ which is not minimal w.r.t. the Hamming distance. Let $\omega_1'$ be the model in $K'$ closest to $\omega_1 = \omega \cup \{a\}$ and $\omega_2'$ the one closest to $\omega_2 = \omega \cup \{b\}$. We need $a\in \omega_1'$ and $b\in \omega_2'$; otherwise $H(\omega,\omega_1')<H(\omega_1,\omega_1')$ or $H(\omega,\omega_2')<H(\omega_2,\omega_2')$;
further we need $b\notin \omega_1'$ and $a\notin \omega_2'$;
otherwise $H(\omega_3,\omega_1')<H(\omega_1,\omega_1')$ or $H(\omega_3,\omega_2')<H(\omega_2,\omega_2')$.
Hence $\omega_1'$ and $\omega_2'$ are incomparable thus also $\omega_1'\cap \omega_2' \in \mod(K')$, since $K'$ is a Horn KB. But then $H(\omega,\omega_1'\cap \omega_2') \leq H(\omega_1,\omega_1')$.
\end{proof}

Our next result shows that
$\Delta^H$ nonetheless increases the range of $\Lhorn$-simplifiable KBs compared to $\Delta^D$ (recall Theorem~\ref{thm:d}).

\begin{proposition}\label{thm:twomods}
Any knowledge base $K$ with $\mod(K)=\{\omega_1,\omega_2\}$
is $\Lhorn$-simplifiable w.r.t.\ $\Delta^{H}$.
\end{proposition}

\begin{proof}
  If $\omega_1,\omega_2$ are comparable, 
	we can apply
  Proposition~\ref{prop:clear}.  Thus, assume $\omega_1,\omega_2$ are
  incomparable and let $d_1=|\omega_1\setminus \omega_2|$ and
  $d_2=|\omega_2\setminus \omega_1|$. W.l.o.g.\ assume $d_1\leq
  d_2$. Also note that $d_1 > 0$.  We use $K'$ with
  $\mod(K')=\{\omega_1^+,\omega_1\cup \omega_2\}$ where $\omega_1^+$
  adds $d_1$ elements from $\omega_2\setminus \omega_1$ to $\omega_1$.
  Thus, $\omega_1^+\subseteq \omega_1\cup \omega_2$ and we can choose
  $K'$ from $\Lhorn$.  Moreover, let $\mu\in\Lhorn$ such that
  $\mod(\mu)=\{\omega_1,\omega_2,\omega_1\cap \omega_2\}$.  We have
  the following distances (note that $d(\omega_2,\omega_1^+)=d_1+(d_2-d_1)$).
$$
\begin{array}{lccc}
           &     \omega_1^+      &        \omega_1\cup \omega_2 &    K'\\
\omega_1          &     d_1      &        d_2      &    d_1 \\
\omega_2          &     d_2  &         d_1     &     d_1\\
\omega_1\cap \omega_2    &     2d_1    &        d_1+d_2    &    > d_1\\
\end{array}
$$
Hence, $\Delta^{H}_\mu(K') \equiv K$ as desired.
\end{proof}

Our final result concerns distributability in the $\horn$ fragment. 
We show that some KBs with three models can be distributed.

\begin{proposition}
Let $K$ be a KB such that $\mod(K) = \{\omega_1,\omega_2,\omega_3\}$. If $\omega_1$, $\omega_2$ and $\omega_3$ are not all pairwise incomparable, then $K$ is $\horn$-distributable w.r.t. $\Delta^{H,\otimes}$ with $\otimes \in \{\sum,\GMax,\GMin\}$.
\end{proposition}

\begin{proof}
  If $K$ is $\horn$-expressible, then the result follows from Proposition~\ref{prop:clear}.
  If $K$ is not $\horn$-expressible, then we do a case analysis on the number of pairwise incomparable models of $K$.

	\textit{Case 1.} 
	If exactly one pair of models of $K$ are incomparable, then we can assume without loss of generality that it is $\omega_1$ and $\omega_2$. 
	It follows then that $\omega_3\neq\omega_1\cap\omega_2$. 
	Also, there must be distinct atoms $a$ and $b$ such that $a\in\omega_1$, $a\notin\omega_2$ and $b\in\omega_2$, $b\notin\omega_1$. 
	We consider a constraint $\mu\in \Lhorn$ such that $\mod(\mu)=\{\omega_1, \omega_2, \omega_3, \omega_1\cap \omega_2\}$.
	
	\textit{Case 1.1.} If $\omega_1 \subseteq w_3$ and $w_2\subseteq \omega_3$, then we take a globally new atom $c$ and KBs $K_1$ and $K_2$ such that:
	\begin{itemize}
	  \item $\mod(K_1)= \{\omega_1\cup \{b\},\omega_2 \cup \{a\},\omega_3, (\omega_1\cap\omega_2)\cup\{a,b\}\}$
	  \item $\mod(K_2)= \{\omega_1,\omega_2, \omega_3\cup\{c\},\omega_1\cap \omega_2\}$	
	\end{itemize}
	It is easy to see that $K_1$ and $K_2$ are $\horn$-expressible. 
	Considering, now, the profile $E=(K_1,K_2)$, we obtain the following distances:
		$$
		\begin{array}{llllll}
		                        &	K_1	&		K_2 & \sum & \GMax  & \GMin\\
		\omega_1        	&	1	&		0   &  1 & (1,0) & (0,1)\\
		\omega_2        	&	1	&		0   &  1 & (1,0) & (0,1)\\
		\omega_3	        &	0	&		1   &   1 & (1,0) & (0,1)\\
		\omega_1\cap \omega_2	&	 2	&	0    & 2    & (2,0) & (0,2)\\
		\end{array}
		$$	
	So for each $\oplus \in \{\sum,GMax,GMin\}$ we obtain that
	$\Delta_\mu^{H,\oplus}(E) \equiv K$.

	\textit{Case 1.2.} If $\omega_3 \subseteq \omega_1$ and $\omega_3 \subseteq \omega_2$, then $\omega_3 \subseteq \omega_1 \cap \omega_2$.
	Moreover, since $\omega_3\neq\omega_1\cap\omega_2$, it actually holds that $\omega_3 \subset \omega_1 \cap \omega_2$.
	Thus there exists an atom $c$ such that $c \in (\omega_1 \cap \omega_2)$ and $c \notin \omega_3$.
	We now take KBs $K_1$, $K_2$, $K_3$ and $K_4$ such that:
	\begin{itemize}
	  \item $\mod(K_1)= \{\omega_1\cup \{b\},\omega_2 \cup \{a\},\omega_3, (\omega_1\cap \omega_2)\backslash\{c\}, (\omega_1\cap\omega_2)\cup\{a,b\}\}$
	  \item $\mod(K_2)= \{\omega_1,\omega_2 \cup\{a\}, \omega_3, (\omega_1\cap \omega_2)\cup \{a\}\}$	
	  \item $\mod(K_3)= \{\omega_1\cup\{b\},\omega_2, \omega_3,(\omega_1\cap \omega_2)\cup\{b\}\}$	
	  \item $\mod(K_4)= \{\omega_1,\omega_2, \omega_3\cup\{c\},\omega_1\cap \omega_2\}$	
	\end{itemize}
	It is easy to see that $K_1$, $K_2$, $K_3$ and $K_4$ are $\horn$-expressible. 
	Considering, now, the profile $E=(K_1,K_2,K_3,K_4,K_4)$, we obtain the following distances:
		
		\begin{table}[h]		
		\centering
		\scalebox{0.82}{
		\begin{tabular}{lllllllll}
											                        & $K_1$ & $K_2$ & $K_3$ & $K_4$ & $K_4$ & $\sum$ & $\GMax$  & $\GMin$\\
		$\omega_1$        							 &	1	  &	   0    &   1    &     0   &     0    &     2       & (1,1,0,0,0) & (0,0,0,1,1)\\
		$\omega_2$        							 &	1	  &	   1    &   0    &     0   &     0   &      2       & (1,1,0,0,0) & (0,0,0,1,1)\\
		$\omega_3$	        						 &	0	  &	   0    &   0    &     1    &    1    &     2       &  (1,1,0,0,0)  &(0,0,0,1,1)\\
		$\omega_1\cap \omega_2$	&	1 	 &    1     &  1      &   0     &	0    &     3       & (1,1,1,0,0) & (0,0,1,1,1)\\
		\end{tabular}
		}
		\end{table}
		
	So for each $\oplus \in \{\sum,GMax,GMin\}$ we obtain that
	$\Delta_\mu^{H,\oplus}(E) \equiv K$.	

	\textit{Case 2.}
	If exactly two pairs of models of $K$ are incomparable, then we can assume without loss of generality that it is $w_1, w_2$ and $w_2, w_3$.
	We consider a constraint $\mu\in \Lhorn$ such that $\mod(\mu)=\{\omega_1, \omega_2, \omega_3, \omega_1\cap \omega_2, \omega_2\cap\omega_3\}$.
	Then there must be distinct atoms $a$ and $b$ such that $a\in \omega_1$, $a\notin\omega_2$ and $b\in\omega_2$, $b\notin\omega_1$.
	Further, there must be distinct atoms $c$ and $d$ such that $c\in \omega_2$, $c\notin\omega_3$ and $d\in\omega_3$, $d\notin\omega_2$.		
	
	\textit{Case 2.1.} If $w_1\subseteq w_3$, then we get that $c\notin\omega_1$ and $a\in\omega_3$.
	We take KBs $K_1$ and $K_2$ such that:
	\begin{itemize}
	  \item $\mod(K_1)= \{\omega_1\cup \{c\},\omega_2,\omega_3 \cup\{c\}, (\omega_1\cap\omega_2)\cup\{c\},(\omega_2\cap\omega_3) \cup \{c\}\}$
	  \item $\mod(K_2)= \{\omega_1,\omega_2\cup \{a\},\omega_3, (\omega_1\cap\omega_2)\cup\{a\}, (\omega_2\cap\omega_3)\cup\{a\}\}$	
	\end{itemize}
	It is easy to see that $K_1$ and $K_2$ are $\horn$-expressible. 
	Considering, now, the profile $E=(K_1,K_2)$ and keeping in mind that $c\notin\omega_1$ and $a\in\omega_3$, we obtain the following distances:
		$$
		\begin{array}{llllll}
		                        &	K_1	&		K_2 &  \sum & \GMax  & \GMin\\
		\omega_1        	&	1	&		0   &  1 & (1,0) & (0,1)\\
		\omega_2        	&	0	&		1   &  1 & (1,0) & (0,1)\\
		\omega_3	        &	1	&		0   &   1 & (1,0) & (0,1)\\
		\omega_1\cap \omega_2	&	1	&	1 &	 2    & (1,1) & (1,1)\\
		\omega_2\cap \omega_3	&	1	&	1 &	 2    & (1,1) & (1,1)\\
		\end{array}
		$$	
	So for each $\oplus \in \{\sum,GMax,GMin\}$ we obtain that
	$\Delta_\mu^{H,\oplus}(E) \equiv K$.

	\textit{Case 2.2.} If $w_3\subseteq w_1$, then we get that $b\notin\omega_3$ and $d\in\omega_1$.
	We take KBs $K_1$, $K_2$ and such that:
	\begin{itemize}
	  \item $\mod(K_1)= \{\omega_1\cup \{b\},\omega_2,\omega_3 \cup\{b\}, (\omega_1\cap\omega_2)\cup\{b\},(\omega_2\cap\omega_3) \cup \{b\}\}$
	  \item $\mod(K_2)= \{\omega_1,\omega_2\cup \{d\},\omega_3, (\omega_1\cap\omega_2)\cup\{d\}, (\omega_2\cap\omega_3)\cup\{d\}\}$	
	\end{itemize}
	It is easy to see that $K_1$ and $K_2$ are $\horn$-expressible. 
	Considering, now, the profile $E=(K_1,K_2)$ and keeping in mind that $b\notin\omega_3$ and $d\in\omega_1$, we obtain the following distances:
		$$
		\begin{array}{llllll}
		                        &	K_1	&		K_2 &  \sum & \GMax  & \GMin\\
		\omega_1        	&	1	&		0   &  1 & (1,0) & (0,1)\\
		\omega_2        	&	0	&		1   &  1 & (1,0) & (0,1)\\
		\omega_3	        &	1	&		0   &   1 & (1,0) & (0,1)\\
		\omega_1\cap \omega_2	&	1	&	1 &	 2    & (1,1) & (1,1)\\
		\omega_2\cap \omega_3	&	1	&	1 &	 2    & (1,1) & (1,1)\\
		\end{array}
		$$	
	So for each $\oplus \in \{\sum,GMax,GMin\}$ we obtain that
	$\Delta_\mu^{H,\oplus}(E) \equiv K$.
	The cases when $\omega_2\subseteq \omega_3$ or $\omega_3\subseteq \omega_2$ are symmetric.
This concludes our case analysis, as any other remaining case results in either all of the interpretations $\omega_1$, $\omega_2$ and $\omega_3$ being pairwise incomparable, or in $K$ being $\horn$-expressible.
\end{proof}

The remaining case (i.e., $\mod(K)=\{\omega_1, \omega_2,\omega_3\}$ with $\omega_1$, $\omega_2$, $\omega_3$ pairwise incomparable), as well as the more general case when $K$ has an arbitrary number of models is subject to ongoing work.

\section{Conclusion}
In this paper we have proposed the notion of distributability and we have studied the properties of several merging operators with respect to different fragments of propositional logic. Our results are summarized in Table~\ref{table:summary}.
\begin{table}[t]
\centering
\begin{tabular}{|c|c|c|c|}
  \hline
  & $\lits$ & $\krom$ & $\horn$ \\ \hline
  simplifiable w.r.t.\ $\Delta^D$ & $\times$ & $\times$ & $\times$ \\ \hline
  simplifiable w.r.t.\ $\Delta^H$ & $\times$ & $\checkmark$ & $\circ$ \\ \hline
  distributable w.r.t.\ $\Delta^{D,\sum}$ & $\checkmark$ & $\checkmark$ & $\checkmark$ \\ \hline
  distributable w.r.t.\ $\Delta^{H,\sum}$ & $\times$ & $\checkmark$ & $-$ \\ \hline
  distributable w.r.t.\ $\Delta^{H,GMax}$  & $-$ & $\checkmark$ & $-$ \\ \hline
  distributable w.r.t.\ $\Delta^{H,GMin}$ & $-$ & $\checkmark$ & $-$ \\ \hline
\end{tabular}
\caption{Summary of Results \label{table:summary}}
\end{table}
The symbol $\times$ means that only ``trivial'' knowledge bases (belonging to the considered fragment) can be distributed with the corresponding operator. 
Alternately, $\checkmark$ means that any knowledge base can be distributed. 
Symbol $-$ means we know that some non-trivial knowledge bases can be distributed, and finally $\circ$ means that some, but not all, non-trivial bases can be simplified. 
Interestingly, the picture emerging from Table~\ref{table:summary} is that merging operators behave quite differently depending on the distance and aggregation function employed, in a way that does not lend itself to simple categorization.
For instance, our results on simplifiability imply that using Dalal revision to $\Llits$ KBs never takes us outside the $\lits$ fragment; applying the same revision operator to $\Lkrom$ KBs can produce any KB in $\L$; and applying it to $\Lhorn$ KBs can produce some, though not all possible KBs.


Several questions are still open for future work. We plan to study the exact characterization of what can (and cannot) be distributed, in order to replace the symbols $-$ and $\circ$ in the previous table. Other merging operators
can also be integrated to our study. Some of our results on distributability require the addition of new atoms to the interpretations. We want to  determine whether similar results can be obtained without modifying the set of propositional variables, in particular for the $\krom$ fragment.
We are also interested in the number of knowledge bases needed to
distribute knowledge: given an integer $n$, a knowledge base $K$ and a
merging operator $\Delta$, is it possible to distribute $K$
w.r.t. $\Delta$ such that the resulting profile contains at most $n$
knowledge bases?
This paper was a first step to understand the limits of distributability; the actual construction of the profile and complexity of this process are important questions that will be tackled in future research.
Finally, we also consider applying the concept of distributability to
non-classical formalisms, in particular in connection with merging 
operators proposed for logic programs  \cite{DelgrandeSTW13}.

\section*{ Acknowledgments}
This work was supported by the Austrian Science Fund (FWF) 
under grant 
P25521.


\end{document}